\documentclass[sigconf,pdftex]{acmart} %

\usepackage{graphicx}
\usepackage{amsmath}
\usepackage{ascmac}
\usepackage{listings}
\usepackage{amsfonts}
\usepackage{bm}
\usepackage{amssymb}

\usepackage{algorithm}
\usepackage{algorithmic}

\usepackage[subrefformat=parens]{subcaption}

\newtheorem{theo}{Theorem}[section]
\newtheorem{defi}{Definition}[section]
\newtheorem{lemm}{Lemma}[section]

 \newtheorem{rem}{Remark}[section]
\newcommand{\COMM}[2]{{
\ifthenelse{\equal{#1}{AT}}{\color{red}}{
\ifthenelse{\equal{#1}{HO}}{\color{blue}}{
\ifthenelse{\equal{#1}{**}}{\color{green}}}}
[#1: #2]
}}

\AtBeginDocument{%
  \providecommand\BibTeX{{%
    \normalfont B\kern-0.5em{\scshape i\kern-0.25em b}\kern-0.8em\TeX}}}

\setcopyright{iw3c2w3}

\copyrightyear{2020}
\acmYear{2020}
\acmConference[WWW '20 Companion]{Companion Proceedings of the Web Conference 2020}{April 20--24, 2020}{Taipei, Taiwan}
\acmBooktitle{Companion Proceedings of the Web Conference 2020 (WWW '20 Companion), April 20--24, 2020, Taipei, Taiwan}
\acmPrice{}
\acmDOI{10.1145/3366424.3383556}
\acmISBN{978-1-4503-7024-0/20/04}

\begin{document}

\title{Convex Fairness Constrained Model \\ Using Causal Effect Estimators}

%\email{webmaster@marysville-ohio.com}
%\affiliation{%
%  \institution{Institute for Clarity in Documentation}
%  \streetaddress{P.O. Box 1212}

\author{Hikaru Ogura}
\email{hikaru\_ogura@mist.i.u-tokyo.ac.jp}
\affiliation{%
\institution{The University of Tokyo}
\country{Japan}
}

\author{Akiko Takeda}
\email{takeda@mist.i.u-tokyo.ac.jp}
\affiliation{%
\institution{The University of Tokyo / RIKEN}
\country{Japan}
}



\begin{abstract}
    Recent years have seen much research on fairness in machine learning.
    Here, mean difference (MD) or demographic parity is one of the most popular measures of fairness.
    However, MD quantifies not only discrimination but also {\em explanatory bias} which is the difference of outcomes
    justified by explanatory features.
    In this paper, we devise novel models, called FairCEEs, which remove discrimination while keeping explanatory bias.
The models are based on estimators of causal effect utilizing propensity score analysis.
%    and which remove discrimination while keeping explanatory bias.
    We prove that FairCEEs with the squared loss theoretically outperform a naive MD constraint model.
    We provide an efficient algorithm for solving FairCEEs in regression and binary classification tasks.
    In our experiment on synthetic and real-world data in these two tasks, FairCEEs outperformed an existing model that considers explanatory bias in specific cases.
\end{abstract}




\maketitle

\section{Introduction}\label{chap:introduction}
%\COMM{AT}{About CCS Concepts: and also check https://arxiv.org/pdf/1908.09635.pdf}
%\COMM{HO}{I added CCS Concepts.}

Recently much research on fairness in machine learning has been conducted. While numerous companies,
scientific communities, and public organizations are collecting huge data repositories,
the data tend to be biased toward an individual or a group based on their inherent or acquired characteristics, which are called {\em sensitive features}, such as race or gender.
Machine learning models trained on such biased datasets may produce discriminatory predictions with respect to the sensitive feature.
The problem of biased datasets is crucial to ensuring fairness when applying machine learning algorithms to applications such as loan assignment \cite{EO1}, criminal risk assessment \cite{criminal_risk}, and school admissions \cite{school_admission}.
It is necessary to develop fair decision-making methods that avoid discrimination arising from biased datasets.

Regarding fairness in machine learning, discrimination is considered
 either direct or indirect \cite{Calders2010}.
Direct discrimination is caused by a sensitive feature such as race or gender, while indirect discrimination is caused by non-sensitive features which are strongly correlated with the sensitive feature.
Removing the sensitive feature from the input removes direct discrimination, but does not remove indirect discrimination.
This is called the red lining effect \cite{Calders2010}.
Furthermore, as described in \cite{Calders2013}, the non-sensitive features that cause indirect discrimination
sometimes include features which cause bias but can be justified.
These features are called explanatory features and the bias caused by the explanatory features are called explanatory bias.
We treat such explanatory bias as non-discriminatory by following the work of \cite{Calders2013}, though
explanatory bias is included in indirect discrimination.
Figure \ref{fig:detail_disc} shows the relationship between discrimination and explanatory bias.

\begin{figure}[t]
\centering
\includegraphics[width=0.8\linewidth]{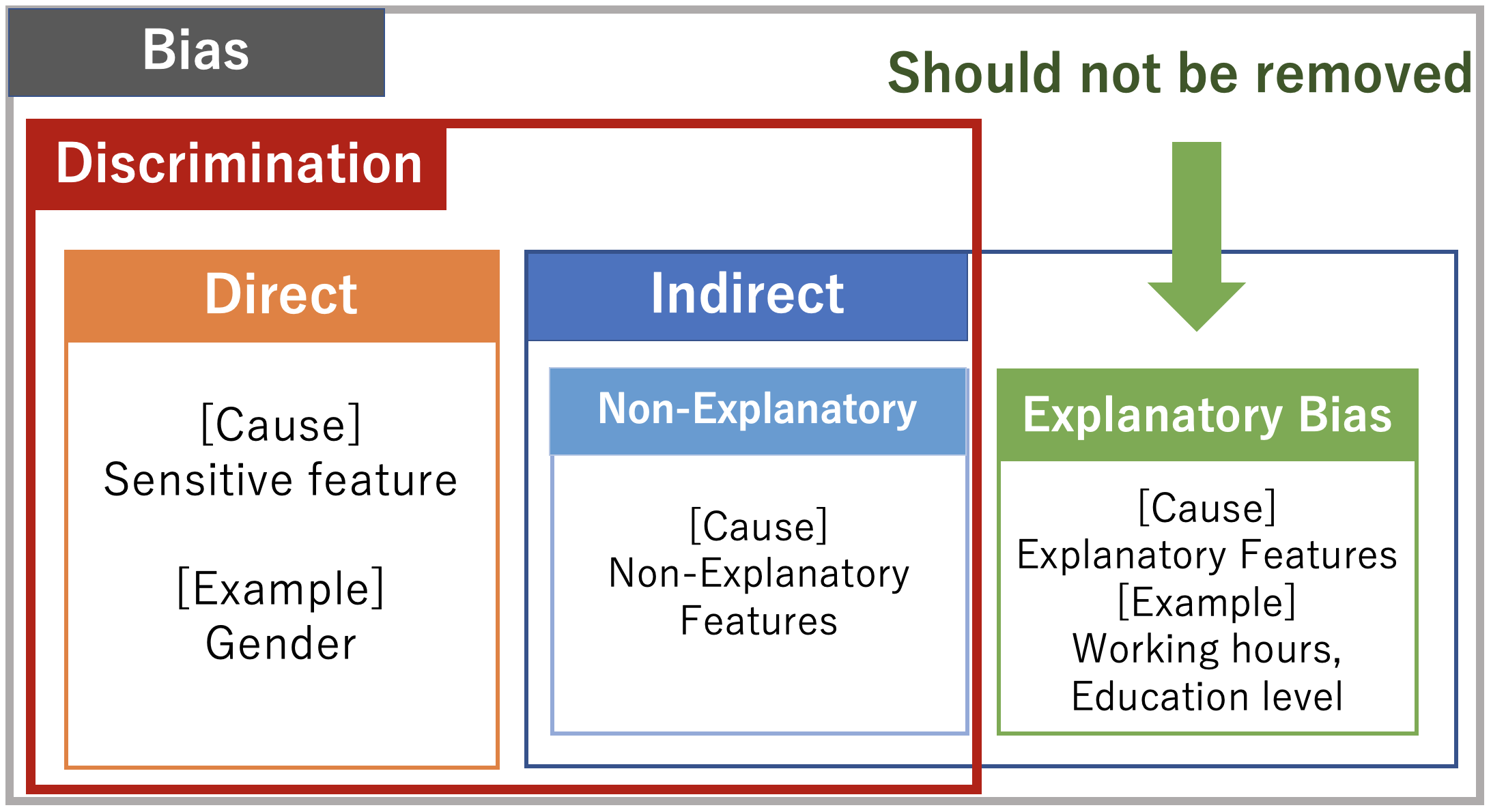}
\caption{Relationship between discrimination and explanatory bias. Sensitive, non-explanatory or explanatory features are selected by the user. The example features are from \cite{Calders2013}, where the task is to predict his/her salary.}
\label{fig:detail_disc}
\end{figure}

% explanatory bias を扱った手法の紹介
Researchers have introduced several statistical criteria \cite{EO1,Calders2010,EO2,individual} in an attempt to formalize notions of fairness. Mean difference (MD) \cite{Calders2010} is one of the most popular measures of fairness to quantify discrimination. This measure is the difference of outcomes between groups distinguished by the sensitive feature. MD can quantify direct and indirect discrimination, but explanatory bias is also included in indirect discrimination. Therefore, when we use MD=0 as a constraint, we may mistakenly remove explanatory bias. This leads to reverse discrimination and degrades performance. Several studies based on MD
 \cite{Calders2010,Kamiran2009,Harrison2015,Zafar2015FairnessCM,Louizos2015} can suffer from this problem.
  In this paper, discrimination is defined as MD except for explanatory bias:
\begin{align}\label{def:def_disc}
    \mbox{Discrim. = MD - explanatory bias}
\end{align}
and use Discrim. =0 as a constraint in our model.
Several studies have dealt with explanatory bias \cite{kamiran2013,Impartial,Kamiran2013ExplainableAN} from a non-causal perspective. In addition to them,
there are mainly two research streams that have used estimators of causal effect \cite{causaleffect,Pearl:2009} instead of MD.
The recent work \cite{Kusner,Kilbertus2018,fair_inference_on_outcomes,Zhang2018,When_Worlds_Collide,path_specificCF,CausalModelingBased,CausalFramework} uses causal diagrams for this purpose,
but the resulting optimization problems tend to be difficult. Indeed, it seems difficult to find a global solution for such problems because of their non-convexity \cite{fair_inference_on_outcomes,EO_causal}.

% この研究での貢献
In this paper, we use propensity score analysis \cite{propensity} for simplifying the models used in other research.
The papers \cite{IPS_discovery,Calders2013} deal with explanatory bias based on propensity score analysis.
In \cite{IPS_discovery} the authors proposed a causal discrimination discovery approach using estimators of causal effect based on the propensity score. Although their work bears similarity to ours in terms of using estimators of causal effect, there are differences in that the estimator in \cite{IPS_discovery} is calculated in its nearest neighbor such as in individual fairness \cite{individual} and in that our model is for preventing discrimination, not for discovery.
In \cite{Calders2013} the authors proposed Multi MD, which removes discrimination and keeps explanatory bias on the basis of stratification \cite{propensity_strat}.
However, unless the number of strata is set appropriately, the stratification may create a stratum that has some explanatory bias, which results in failure of Multi MD.

We propose new models, called Fair prediction using Causal Effect Estimators, (FairCEEs) that use another estimator based on the propensity score,
the IPW \cite{IPW} or DR \cite{DR} estimator, for causal effect.
Using the estimators for a constraint instead of MD makes it possible to remove the discrimination and keep the explanatory bias.

Our contributions are as follows:
\begin{itemize}
\item
  FairCEEs based on the IPW and DR estimators of causal effect are formulated as simple convex quadratic optimizations; for classification, they are globally solved with the proximal gradient method; for regression, they are solved using the Lagrange multiplier method.
\item FairCEE-IPW and Multi MD theoretically outperform Single MD in regression tasks.
\item Numerical experiments on synthetic and real-world data show that FairCEEs outperform Multi MD in terms of fairness and prediction accuracy.
\end{itemize}
The organization of the paper is as follows: after preliminaries are introduced in Section~\ref{chap:preliminaries}, our model is described in Section \ref{chap:proposed_model}. The algorithm and theoretical analysis are in Section \ref{chap:algo_analysis}. Section \ref{chap:experiments} describes numerical experiments on synthetic and real-world datasets. A conclusion follows.

\section{Preliminaries}\label{chap:preliminaries}

We review MD \cite{Calders2010}, which is one of the most popular measures of fairness, and
show an existing model Multi MD \cite{Calders2013}, which removes discrimination %defined as
\eqref{def:def_disc}.
Next, we point out the disadvantages of Multi MD. Finally, we describe the potential outcome model \cite{rubinmodel} and its relationship with Multi MD.
    \subsection{Notation}\label{sec:notation}
    Suppose that $N$ and $d$ are the number of data and non-sensitive features, respectively and one sensitive feature
%    \COMM{AT}{multiple sensitive features are not considered??}
    exists in a given dataset.
%    \footnote{Although we consider the setting with one binary sensitive feture in this paper, all results can be generalized to multiple sensitive features or a sensitive feature with multiple levels (i.e. $S \in \{ 0, 1, 2,...N_s\}$).}
%    \COMM{HO}{sensitive featureを複数にする、複数クラスにするという二種類の拡張はそれぞれについて因果効果を考えることでできそうだと思う(脚注に追加しました)のですが、断言するのは少し怖いです。}.
    $\bm X \in \mathbb{R}^{N\times d}$ is a matrix consisting of non-sensitive feature values.
    $\bm x_i \in \mathbb R^{d}$ indicates the feature values of the $i$-th instance, i.e., the $i$-th row of $\bm X$, and $S$ denotes the sensitive feature.
    $S$ is a binary variable such as male/female, denoted by $S \in \{+1, 0\}$.
    $s_i$ is the observed value of sensitive feature $S$ of the $i$-th instance, and $\bm s$ is the vector defined by $\bm s = (s_1,s_2,\ldots,s_N)^\top$.
          For convenience,
       we define two index sets $\mathcal{I}^+$ and $\mathcal{I}^-$, defined by $\mathcal{I}^+ = \{ i \ | i \in [1, N] , s_i = 1\}$ and $\mathcal{I}^- = \{ i\ | i \in [1, N] , s_i = 0\}$.
      $Y$ is the target label; in a regression task, $Y \in \mathbb R$, and in a binary classification task, $Y \in \{+1,0\}$, unless otherwise specified. $y_i$ is the observed label of the $i$-th instance, and $\bm y$ is the vector defined by $\bm y = (y_1,y_2,\ldots,y_N)^\top$.
      $\hat y_i$ is the predicted label of the $i$-th instance.
      In this paper, we consider a linear model class and denote the trainable parameter by $\bm w \in \mathbb R^{d}$, i.e,  $\hat y_i = \bm x_i \bm w$ for regression or $\hat y_i = \sigma (\bm x_i \bm w)$ for classification where $\sigma (\cdot)$ is a sigmoid function.
      %       In a regression task, $Y \in \mathbb R$, and in a binary classification task, $Y \in \{+1,0\}$, unless otherwise specified.
%      For convenience,
%       we define two index sets $\mathcal{I}^+$ and $\mathcal{I}^-$, defined by $\mathcal{I}^+ = \{ i \ | i \in [1, N] , s_i = 1\}$ and $\mathcal{I}^- = \{ i\ | i \in [1, N] , s_i = 0\}$.

        In what follows, we will consider the supervised learning setting, where we have a dataset ($\bm X$, $\bm s$, $\bm y$) meaning (non-sensitive feature values, sensitive feature values, labels),
        and our goal is to train a model whose prediction does not include discrimination \eqref{def:def_disc},
        but has good accuracy.
    \subsection{Single MD and Multi MD} \label{subsec : single_multi}
    MD is a measure quantifying how discriminatory the dataset is for classification and regression tasks.

    \begin{defi}[Mean Difference \cite{Calders2010}]\label{def:md}

        \begin{align}\label{def:md_math}
            \mbox{MD}=\frac{1}{|\mathcal{I}^+|}\sum_{i \in \mathcal{I}^+} y_i - \frac{1}{|\mathcal{I}^-|}\sum_{i \in \mathcal{I}^-} y_i,
        \end{align}
    where $|\mathcal{I}|$ is the size of the set $\mathcal{I}$. Note that by replacing $y_i$ with $\hat y_i$ in the definition, MD represents how discriminatory the prediction is.
    \end{defi}

     Now let us describe Single MD and Multi MD \cite{Calders2013}, which are related to our model. \footnote{Single MD and Multi MD are denoted by SEM-S, SEM-MP in \cite{Calders2013}.} Single MD is formulated as
        \begin{align}\label{prob:single}
             & \min_{\bm w} && \|\bm {Xw}-\bm y\|^2  \notag \\
          & \mathrm{s.t.} &&  \bm d^\top \bm{Xw}=0,
        \end{align}
        where

        \begin{align} \label{def:md_constraint}
            \bm d= \frac{\bm s}{\bm1^\top \bm s}-\frac{\bm 1 - \bm s}{\bm 1^\top (\bm 1 - \bm s)}.
        \end{align}

        From \eqref{def:md_math} and \eqref{def:md_constraint} and reminding that $s_i=1$ for $i \in \mathcal{I}^+$ and $s_j=0$ for $j \in \mathcal{I}^-$,
         $\bm d^\top \bm{Xw}$ represents MD of outputs,
        and therefore, the constraint in \eqref{prob:single} means MD=0. As mentioned in Section \ref{chap:introduction}, Single MD does not consider explanatory bias; it is defined by \cite{Calders2013} as the differences in outcomes of groups distinguished by a sensitive feature value that is explainable by explanatory features and can be justified.
        The explanatory features are denoted by $\bm X_e \in \mathbb R^{d_e}$,
        where $d_e$ is the number of explanatory features, while the non-explanatory feature values are denoted by $\bm X_n \in \mathbb R^{d-d_e}$.
        In our settings, the explanatory features are provided externally.

        \begin{rem}
          It seems possible to remove non-explanatory discrimination without removing explanatory bias by using MD in an inequality constraint, i.e., $\bm d^\top \bm{Xw} \le \delta$ with a hyperparameter $\delta$,
          following \cite{komiyama,Zafar2015FairnessCM,Agarwal}
           or in a regularizer, i.e., $\lambda(\bm d^\top \bm{Xw})^2$ with a hyperparameter $\lambda$, as in \cite{Kamishimab,Berk2017,Zemel}.
          However, these approaches require the correct explanatory bias in the hyperparameter $\delta$, $\lambda$, but the true value of the explanatory bias is unknown.
          It is difficult to determine appropriate hyperparameter values because the criteria used in validation % to find the hyperparameters
          for hyperparameters %suitable for the explanatory bias
          is not clear.
          On the other hand, our approaches described later can remove discrimination while keeping explanatory bias without using hyperparameters.
        \end{rem}

        Calders et al. \cite{Calders2013} proposed Multi MD which utilizes stratification \cite{propensity_strat}
        according to a propensity score \cite{propensity} for a constraint.

        \begin{defi}[Propensity Score \cite{propensity}]
        Let $S$ and $\bm X_e$ be the sensitive feature and the explanatory features, respectively.
        The propensity score is defined as
            \[
                Z=P(S=1|\bm X_e).
            \]
        \end{defi}

        In general, because the true propensity score is unknown, we have to estimate it from the dataset. Usually this is done using a logistic regression model.
        Using this propensity score and stratification technique \cite{propensity_strat}, Multi MD can be formulated as follows \cite{Calders2013}. First, we estimate propensity scores.
        Then, we split the dataset into $K$ strata via the propensity score quantiles\footnote{Although $K$ is fixed to five in \cite{Calders2013}, we formulate the generalized form with various values of $K$ in this paper.}.
        Next, we define Multi MD constraints as MD=0 in each stratum.
        Formally $\bm X_{(k)}$, $\bm y_{(k)}$, $\bm s_{(k)}$ and $\bm d_{(k)}$ are feature values, labels, sensitive feature values and constraint vector
        respectively, in the $k$-th stratum.
        \begin{align}\label{prob:multi}
            & \min_{\bm w}          &&  \sum_{k=1}^K \|\bm X_{(k)} \bm w-\bm y_{(k)}\|^2  \notag \\
            & \mathrm{s.t.} &&  \bm d_{(k)}^\top \bm X_{(k)} \bm w=0 \ (k=1,2,...,K),
        \end{align}
        where
        \[
            \bm d_{(k)}=\frac{\bm s_{(k)}}{\bm1^\top \bm s_{(k)}}-\frac{\bm 1 - \bm s_{(k)}}{\bm 1^\top (\bm 1 - \bm s_{(k)})}.
        \]

        In the same way as Single MD \eqref{prob:single}, $\bm d_{(k)}^\top \bm X_{(k)} \bm w$ represents the value of MD of outputs in the $k$-th stratum.

    \subsection{Disadvantages of Multi MD}\label{subsec : disadvantages}
        Multi MD removes Discrim. \eqref{def:def_disc} by using stratification with respect to the propensity score. Here, we need an appropriately large number of strata $K$, because if $K$ is too small, each stratum will have a range of propensity scores and some explanatory bias may remain in each stratum. Unfortunately, there are three cases where we can not increase the number of strata $K$ appropriately.
        In this paper, we refer these cases as Imbalance CASE, Degenerate CASE and Inferred CASE.

        Imbalance CASE : the case where datasets are imbalanced with respect to the sensitive feature value of 0 or 1 (e.g. $|\mathcal I^+|:|\mathcal I^-|=9:1$).
         Here, splitting the dataset into many strata may create a stratum that has only a one-sided sensitive feature value.
         In that case, we can not calculate the MD, which leads to a failure of Multi MD. Since this case is likely to appear in fairness settings, it is a serious weakness of Multi MD.

        Degenerate CASE : the case where the number of features $d$ is small.
        Here, setting $K > d$ leads to degeneration and the optimal solution of Multi MD \eqref{prob:multi} becomes the zero vector,
        which satisfies MD=0. This means that Multi MD with large $K$ removes explanatory bias incorrectly and degrades performance in this case.

        Inferred CASE : the case where the estimated value of propensity score $z_i$ is close to $s_i$.
        As well as Imbalance CASE, splitting the dataset into many strata may create a stratum that has only a one-sided sensitive feature value in this case.

    \subsection{Potential Outcome Model (POM)} \label{subsec : potential_outcome_model}
        Now let us describe the potential outcome model (POM) \cite{rubinmodel}.
        Although the relationship between POM and Multi MD is not mentioned in \cite{Calders2013}, we should point out that the discrimination \eqref{def:def_disc} referred to in \cite{Calders2013} corresponds to the causal effect defined by POM. In Multi MD, the model estimates the causal effect by stratification \cite{propensity_strat} and uses the estimated value for its constraint.
        Our model uses POM in a similar way.

        POM considers potential outcome variables denoted by $Y_1$ and $Y_0$.
        Target label $Y$ has both the potential outcome $Y_1$ that would have resulted if had received the sensitive feature $S=1$,
        and the potential outcome $Y_0$ that would have resulted if had received the sensitive feature $S=0$. %$Y_1$ is $Y$ if $S$ is 1. $Y_0$ is defined similarly.
        We can observe either of them, but not both. In other words, we can observe $Y_1$ but not $Y_0$ when $S$ is equal to 1.
        Formally, we can write $Y=SY_1+(1-S)Y_0$.
        By using these variables, causal effect \cite{causaleffect} is defined as the difference between potential outcome variables.
        \begin{align}\label{def:causaleffect}
            \mbox{CE} = E[Y_1]-E[Y_0]
        \end{align}
        Since we can observe either of them, but not both, estimating the causal effect is a missing data problem.

        MD is generally not a consistent estimator of causal effect; i.e.,
       \[
            \mbox{MD} \underset{n\to \infty}{\to}E[Y_1|S=1]-E[Y_0|S=0] \neq E[Y_1]-E[Y_0].
        \]
        because of covariates that are correlated with both $S$ and $Y$.
%        \COMM{HO}{unbiased ではなくconsistentかどうかの方がいいかと思いました。(IPWやDRはCEの一致推定量なので)}
         %(i.e $E[\mbox{MD}]=E[Y|S=1]-E[Y|S=0] \neq \mbox{CE}$).
%        Formally, we have
%        \[
%            \mbox{MD} \underset{n\to \infty}{\to}E[Y_1|S=1]-E[Y_0|S=0] \neq E[Y_1]-E[Y_0].
%        \]
        If we regard the explanatory features as covariates,
        then the causal effect corresponds to Discrim. \eqref{def:def_disc}.
        This is the discrimination we want to remove. When we consider the discrimination as the causal effect \eqref{def:causaleffect}, the true explanatory bias corresponds to
        \[
        \underbrace{(E[Y_1|S=1]-E[Y_0|S=0])}_{\mbox{estimated by MD}} - \underbrace{(E[Y_1]-E[Y_0])}_{\mbox{Discrimination}}.
        \]
        There are several methods of estimating causal effect by using a propensity score, including stratification \cite{propensity_strat} used in \cite{Calders2013}.
        In this paper, instead of stratification, we use IPW and DR estimators, which are explained in Section \ref{chap:proposed_model}.

\section{Proposed Model}\label{chap:proposed_model}

Now we propose new models, called Fair prediction using Causal Effect Estimators (FairCEEs). The estimators, IPW and DR, used for FairCEEs are defined below.

We can avoid the three CASEs described in Section \ref{subsec : disadvantages} utilizing these estimators because they do not need dividing datasets into some strata.
Although the sensitive feature $S$ is binary for simplicity, these models can be generalized to multiple groups.

    \subsection{FairCEE-IPW}\label{sec: ipw_model}
    Here, we introduce the Inverse Probability Weighting (IPW) estimator \cite{IPW} and give a formulation using it for its constraint.
     Note that we use the loss function $\mathcal L(\bm w)$ as the objective function in the following formulations.
     We can apply our method to regression and binary classification tasks by setting $\mathcal L(\bm w)$ to the squared loss ($|\bm {Xw}-\bm y\|^2$) or logistic loss, respectively \footnote{Note that
       the input $\bm s$ of a sensitive feature is not used in the prediction models.}.

    \begin{defi}[IPW Estimator \cite{IPW}]\label{def:ipw}
        %Here, $z_i$ is the propensity score for the $i$-th instance, and $y_i$ is the label for the $i$-th instance. Accordingly, we define the IPW estimator of causal effect as
        Here, $z_i$ is the propensity score for the $i$-th instance. Accordingly, we define the IPW estimator of causal effect as

        \begin{align}\label{def:ipw_math}
            \mbox{IPW} = \frac{\sum_{i=1}^N \frac{s_i}{z_i}y_i}{\sum_{i=1}^N \frac{s_i}{z_i}} - \frac{\sum_{i=1}^N \frac{1-s_i}{1-z_i}y_i}{\sum_{i=1}^N \frac{1-s_i}{1-z_i}}.
        \end{align}

    \end{defi}
    The IPW estimator is a consistent estimator of causal effect \eqref{def:causaleffect} under the assumption that the modeling for the propensity score is correct
    \footnote{We can not confirm that the assumption holds because either of potential outcome variables is missing. However, in practice, we can check some metrics which measure the goodness of fit of the model such as AUC.}.
    Accordingly, FairCEE-IPW is formulated as
        \begin{align}\label{prob:ipw}
            & \min_{\bm w}          &&  \mathcal L(\bm w) \notag \\
            & \mathrm{s.t.} &&  \bm h^\top\bm{Xw}=0,
        \end{align}
        where

        \begin{align}\label{def:ipw_constraint}
            \bm h = \frac{\bm a}{\bm1^\top \bm a}-\frac{\bm b}{\bm 1^\top \bm b},
            \bm a=(\frac{s_1}{z_1},...,\frac{s_N}{z_N})^\top,
            \bm b=(\frac{1-s_1}{1-z_1},...,\frac{1-s_N}{1-z_N})^\top.
        \end{align}

        By the definition \eqref{def:ipw_math} and \eqref{def:ipw_constraint},
        $\bm h^\top\bm{Xw}$ is the value of the IPW estimator of outputs (i.e. the constraint in \eqref{prob:ipw} means IPW=0.).

    \subsection{FairCEE-DR}\label{sec: dr_model}
        The other estimator of causal effect is the Doubly Robust (DR) estimator \cite{DR}.

        \begin{defi}[DR Estimator \cite{DR}] \label{def:dr}
            Here, $z_i$ is the propensity score of the $i$-th instance. $g^+_{i}$ and $g^-_{i}$ are predictions of potential outcome variables ($Y_1, Y_0$) of the $i$-th instance. Accordingly, we define DR estimator of causal effect as
            \[
            \mbox{DR}= \frac{1}{N}\sum_{i=1}^N[(\frac{s_i}{z_i} -\frac{1-s_i}{1-z_i} )y_i+(1-\frac{s_i}{z_i})g^+_{i} -(1-\frac{1-s_i}{1-z_i})g^-_{i}].
            \]
        \end{defi}
        In addition to modeling the propensity score, the DR estimator uses the modeling for $Y_1$ and $Y_0$. Since either of the true values of $Y_1$ and $Y_0$ are unknown, we use the estimated values, i.e., $g^+_{i}$ and $g^-_{i}$, instead of $Y_1$ and $Y_0$. In this study, we train two models $G^+$ and $G^-$ for $Y_1$ and $Y_0$, respectively. $G^+$ is trained using explanatory features whose index belongs to $\mathcal I^+$ and $G^-$ is trained using explanatory features whose index is in $\mathcal I^-$. We use the output of $G^+$ and $G^-$ as $g^+_{i}$ and $g^-_{i}$
        instead of $Y_1$ and $Y_0$ (i.e. $g^+_{i}=G^+(\bm x_{i,e})$ and $g^-_{i}=G^-(\bm x_{i,e})$ where  $x_{i,e}$ is the explanatory feature values of the $i$-th instance.)

        We explained above that if the propensity score is incorrectly estimated, the IPW estimator will not be consistent.
        For this reason, the IPW estimator is not consistent in Inferred CASE. %(in detail, the IPW estimator is close to MD in Inferred CASE).
        However, it is known that even if the estimated propensity score is wrong, the DR estimator remains a consistent estimator of causal effect as long as $Y_1$ and $Y_0$ are estimated correctly (see, e.g., \cite{DR}).
        We formulate FairCEE-DR as follows. (Note that $\bm a$ and $\bm b$ are defined in \eqref{def:ipw_constraint}.)
        \begin{align}
            \label{prob:dr}
            & \min_{\bm w}          &&  \mathcal L(\bm w) \notag \\
            & \mathrm{s.t.} &&  (\bm a-\bm b)^\top\bm {Xw}+(\bm 1 -\bm a)^\top\bm g^+-(\bm 1 -\bm b)^\top\bm g^-=0
        \end{align}
        where
        \[
            \bm g^+=(g^+_{1}, g^+_{2},..., g^+_{N})^\top, ~~
            \bm g^-=(g^-_{1}, g^-_{2},...,g^-_{N})^\top.
        \]

        As well as FairCEE-IPW, $(\bm a-\bm b)^\top\bm {Xw}+(\bm 1 -\bm a)^\top\bm g^+-(\bm 1 -\bm b)^\top\bm g^-$ is the estimated value of DR estimator of the prediction (i.e. the constraint of \eqref{prob:dr} means DR=0.).
        Since the constraints of FairCEEs with a linear model are linear in terms of $\bm w$, we can solve them using algorithms described below\footnote{In general, the constraints of FairCEEs are linear with respect to outputs. Applying FairCEEs to a non-linear model may require another algorithm.}.

\section{Algorithm and Analysis}\label{chap:algo_analysis}

In this section, we describe the algorithms for FairCEEs in regression and classification tasks.
Then we show two theorems on the squared losses of FairCEE-IPW and Multi MD.

\subsection{Algorithm for Regression and Classification}
In a regression task, we use the squared loss function $\|\bm{Xw} -\bm y\|^2$ as $\mathcal L(\bm w)$.
We can solve FairCEEs with the squared loss by using the Lagrange multiplier method as is done with Multi MD in \cite{Calders2013}.
%Specifically, we obtain an optimal solution of FairCEEs with the squared loss $\mathcal L(\bm w)$ for regression by solving a system of linear equations.
    % \section{Algorithm for Binary Classification}

In a binary classification task, we use the logistic loss function $\sum_{i=1}^N \log (1+e^{-y_i \bm w^\top\bm x_i})$ as $\mathcal L(\bm w)$ in FairCEEs \eqref{prob:ipw} and \eqref{prob:dr}.
Then we solve them with the proximal gradient method (PGM) (see, e.g., \cite{Book:Beck:2017}).
FairCEEs, Single MD and Multi MD can be formulated in the following generalized form\footnote{Note that our algorithm can be applied to not only FairCEEs but also Single MD and Multi MD for binary classification.}.
Let $\bm P \in \mathbb{R}^{d\times m}$ and $\bm q \in \mathbb{R}^{m}$ be a constant matrix and vector, respectively,  for constraints.
$m$ is the number of constraints. In Multi MD, $m$ is equal to $K$. In the other methods, $m$ is equal to 1.
Let $y_i \in \{+1,-1\}$ be the binary label for the $i$-th instance.
\begin{align}\label{prob:generalized}
    & \min_{\bm w}          &&  \sum_{i=1}^N \log (1+e^{-y_i \bm w^\top\bm x_i}) \notag \\
    & \mathrm{s.t.} &&  \bm{Pw}=\bm q.
\end{align}

%    To apply the proximal gradient method, we rewrite it as follows.
% \begin{align}
%        & \min_{\bm w}          &\ \  \sum_{i=1}^N \log (1+e^{-y_i \bm w^\top\bm x_i}) + \iota_{\bm P,\bm q}(\bm w),
%    \end{align}
%    where
%    \[
%    \iota_{\bm P,\bm q}(\bm w) =
%    \begin{cases}
%        0  \ \mathrm{if} \  \bm{Pw}=\bm q,  \\
%        +\infty \ \mathrm{otherwise}.
%      \end{cases}
%    \]
    Next, we apply the proximal gradient method to \eqref{prob:generalized}.
    Algorithm~\ref{alg:PGM} is an overview of our algorithm\footnote{%Setting $\eta$ to $1/L$, i.e., setting a small step size, may lead to a long computation.
    In our algorithm, we use a backtracking line search with a clipping process ($\eta \leftarrow \max{(\beta \eta,\frac{1}{L})}$),
     which prevents the step size from being less than $1/L$.},
    where $\hat {\mathcal L}_{\frac{1}{L}}(x;y):=\mathcal L(y)+\nabla \mathcal L(y)^\top(x-y)+\frac{L}{2}\|x-y\|^2$.

         The following theorem implies that the proximal gradient method, Algorithm~\ref{alg:PGM}, has a convergence rate of $O(1/t)$.
         \begin{algorithm}
         \caption{PGM for FairCEEs} %linear constrained models}
         \label{alg:PGM}
         \begin{algorithmic}
         \REQUIRE $\bm P, \bm q, \bm X, \bm y, \eta_0, 0<\beta<1$
         \STATE $\eta = \eta_0$
         \STATE Initialize $\bm w_0$
         \FOR{$t=1,2,\ldots$}
             \WHILE{True}
             \STATE $\bm w'_{temp} \leftarrow \bm w_{t-1} -\eta \nabla \mathcal L(\bm w_{t-1})$
             \STATE $\bm w_{temp} \leftarrow \bm w'_{temp} - \bm P^\top(\bm{PP}^\top)^{-1}(\bm {P\bm w}'_{temp}-\bm q)$
                 \IF{$\mathcal L(\bm w_{temp}) \le \hat {\mathcal L}_{\eta}(\bm w_{temp};\bm w_{t-1})$}
                    \STATE Keep $\bm w_{temp}$ as $\bm w_{t}$
                    \STATE Break
                 \ELSE
                    \STATE $\eta \leftarrow \max{(\beta \eta,\frac{1}{L})}$
                 \ENDIF
             \ENDWHILE
         \ENDFOR

         \end{algorithmic}
         \end{algorithm}

         \begin{theo}\label{theo:convergence}
           PGM for FairCEEs \eqref{prob:ipw} and \eqref{prob:dr} satisfies
           \[
           f(\bm{w}_t)-f^* \leq \frac{L}{2t}\|\bm{w}_t-\bm{w}^*\|_2^2,
           \]
           where $\bm{w}^*$ is an optimal solution of FairCEEs and the objective value $f^*$.
           \end{theo}
           \begin{proof}
                 With the help of references such as
               \cite {Book:Nesterov:2004},
                 the statement can be proved if the gradient vector of $\mathcal L$  is Lipschitz continuous
                 with a constant $L$
               (called {\it $L$-smooth});
               i.e., there exists a constant $L>0$ that satisfies
               \[
               \|\nabla \mathcal L(\bm{w})-\nabla \mathcal L(\bm{v})\|_2\le L\|\bm{w} -\bm{v}\|_2\quad (\bm{w},\bm{v}\in \mathbb R^{d}).
               \]
%               Actually, the logistic loss function has the $L$-smoothness.

%% L-smoothの証明
    %       Indeed, the logistic loss function is $L$-smooth with the constant $L$  equal to $\sum_{i=1}^N\|\bm x_i\|^2$.
        Next, we prove that the logistic loss function is $L$-smooth function and $L=\sum_{i=1}^N\|\bm{x}_i\|^2$.

               First, we show that $g(t)=\log(1+e^t)$ is an $L$-smooth function.
               Since $g^{\prime}(t)=e^t/(1+e^t)=1/(1+e^{-t})$, and $0 \le g^{\prime\prime}(t) \le 1$, we can say the inequality
               $\|g^{\prime}(a)-g^{\prime}(b)\|\le \|a-b\|$, which means that $g(t)$ is an $L$-smooth function with $L$ equal to 1.

               %In fact, $g^{\prime}(t)$ is a sigmoid function. Therefore, this inequality is obvious.
               Next, we consider the following logistic loss function.
               \[
                   F(\bm w;\bm X, \bm y)=\sum_{i=1}^N \log(1+e^{-y_i\bm w^\top\bm x_i})=\sum_{i=1}^N f(\bm w;\bm x_i, y_i).
               \]
              We define $h(\bm w;\bm x,y)=-y\bm w^\top\bm x$.
              %Then we have $f(\bm w;\bm x,y)=g(h(\bm w;\bm x,y))$.
              The  derivative is $\nabla f(\bm w;\bm x,y) = -yg'(h(\bm w;\bm x))\bm x$  by the chain rule.

%               \begin{eqnarray*}
%                       \nabla f(\bm w;\bm x,y) &=& g'(h(\bm w;\bm x,y))\nabla h(\bm w;\bm x,y)\\
%                                    &=& -yg'(h(\bm w;\bm x))\bm x
%               \end{eqnarray*}
           Accordingly, we can obtain the following inequalities:

               \begin{eqnarray*}
                   && \|\nabla F(\bm\alpha;\bm X, \bm y)-\nabla F(\bm\beta;\bm X, \bm y)\|\\
                %    &=& \|\sum_{i=1}^N(\nabla f(\bm\alpha;\bm x_i,y_i)-\nabla f(\bm\beta;\bm x_i,y_i))\|\\
                                                       &\le& \sum_{i=1}^N\|\nabla f(\bm\alpha;\bm x_i,y_i)-\nabla f(\bm\beta;\bm x_i,y_i)\|\\
                                                       &=& \sum_{i=1}^N\|y_i g'(h(\bm \beta;\bm x_i,y_i))\bm x_i-y_i g'(h(\bm \alpha;\bm x_i,y_i))\bm x_i\|\\
                                                       &=& \sum_{i=1}^N |g'(h(\bm \beta;\bm x_i,y_i))\bm x_i-g'(h(\bm \alpha;\bm x_i,y_i))| \|\bm x_i\| (\because |y_i| =1) \\
                                                       &\le& \sum_{i=1}^N |h(\bm \beta;\bm x_i,y_i))- h(\bm \alpha;\bm x_i,y_i))| \|\bm x_i\|\\ %(\because g \text{ is } L\text{-smooth with } L=1.) \\
                    %                                   &=& \sum_{i=1}^N |y_i\bm \alpha^\top \bm x_i - y_i\bm \beta^\top \bm x_i| \|\bm x_i\|\\
                    %                                   &=& \sum_{i=1}^N |(\bm \alpha -  \bm \beta)^\top \bm x_i| \|\bm x_i\| (\because |y_i| =1)\\
                                                       &\le&  (\sum_{i=1}^N  \|\bm x_i\|^2)  \|\bm \alpha -  \bm \beta\|
              \end{eqnarray*}

               This means that the logistic loss $F(\bm w;\bm X, \bm y)$ is an $L$-smooth function and $L=\sum_{i=1}^N  \|\bm x_i\|^2$.

           \end{proof}

%         \begin{proof}
%           With the help of references such as
%\cite {Book:Nesterov:2004},
%           the statement can be proved if the gradient vector of $\mathcal L$  is Lipschitz continuous
%           with a constant $L$
%   (called {\it $L$-smooth});
%i.e., there exists a constant $L>0$ that satisfies
%\[
%\|\nabla \mathcal L(\bm{w})-\nabla \mathcal L(\bm{v})\|_2\le L\|\bm{w} -\bm{v}\|_2\quad (\bm{w},\bm{v}\in \mathbb R^{d}).
%\]
%  Indeed,  $\mathcal L$ is $L$-smooth with the constant $L$  equal to $\sum_{i=1}^N\|\bm x_i\|^2$.
%    The details are in Appendix \ref{chap: l-smooth}.
%\end{proof}
%      This theorem implies that the proximal gradient descent has a convergence rate of $O(1/t)$.
     \subsection{Theoretical Analysis}
Now we show some theoretical guarantees for the proposed model, FairCEE-IPW \eqref{prob:ipw}.
The following theorem says %that the optimal value of the squared loss of FairCEE-IPW \eqref{prob:ipw} is less than that of Single MD \eqref{prob:single} under some reasonable assumption.
%This implies
that our model FairCEE-IPW can achieve a smaller training error than  Single MD \eqref{prob:single} under a reasonable assumption.
The assumption in the theorem is that
 the value estimated by the IPW estimator on a dataset (i.e. $\bm h^\top \bm y$)
     is less than the value estimated by MD ($=\bm d^\top \bm y$), and the assumption is likely to hold as discussed in the end of this section.
%     the optimal value of the squared loss of FairCEE-IPW \eqref{prob:ipw} is less than that of Single MD \eqref{prob:single}.

%     Next, we show the implications of Theorem \ref{theo:ipw_loss} and \ref{theo:multi_loss}. %, whose proofs are given in Appendix \ref{chap:proof_ipw_loss} and \ref{chap:proof_multi_loss}.
     \begin{theo}[FairCEE-IPW Loss]\label{theo:ipw_loss}
         Let $\bm w_{\mathrm{single}}^*$, $\bm w_{\mathrm{ipw}}^*$  be optimal solutions of Single MD \eqref{prob:single} and FairCEE-IPW \eqref{prob:ipw}.
         Suppose that
         \[
            0 \le \bm h^\top \bm y \le \bm d^\top \bm y.
         \]
         Then we can obtain the following inequality:
         \[
             \|\bm X \bm w_{\mathrm{ipw}}^*-\bm y\|^2 \le \|\bm X \bm w_{\mathrm{single}}^*-\bm y\|^2.
         \]
     \end{theo}

     To prepare for this proof, we first must prove the following Lemma \ref{lemma:optimal_loss} and \ref{lemma:md_norm}.

     \begin{lemm}[Loss of optimal solution]\label{lemma:optimal_loss}
         Consider an optimization problem formulated as follows.
         \begin{align*}
             & \min_{\bm w} \ \  \|\bm {Xw}-\bm y\|^2 \\
             & \mathrm{s.t.} \ \ \ \ \bm c^\top \bm{Xw}=0
         \end{align*}

         Let $w^{*}$ be the optimal solution of the above optimization problem.
         Then, the value of loss function is obtained as follows when $\bm w$ is optimal (i.e $\bm w=\bm w^*$)
         \[
             \|\bm X \bm w^*-\bm y\|^2=\frac{(\bm c^\top \bm y)^2}{\|\bm c\|^2}.
         \]
     \end{lemm}

     \begin{proof}\label{proof:optimal_loss}
         The optimal solution $\bm w^{*}$ satisfies the constraint $\bm c^\top \bm X \bm w^{*}=0$.
         Then, we have
         \[
             \bm c^\top (\bm X \bm w^*-\bm y) = -\bm c^\top \bm y.
         \]
         We get
         \[
             \|\bm X \bm w^*-\bm y\|^2= \frac{(\bm c^\top \bm y)^2}{\|\bm c\|^2 \cos^2 \theta},
         \]
         where $\theta$ is the angle between $\bm c$ and $\bm X \bm w^*-\bm y$.

         Here, since $\bm w^*$ is the optimal solution, we can say
         \[
             \|\bm X \bm w^*-\bm y\|^2=\frac{(\bm c^\top \bm y)^2}{\|\bm c\|^2}.
         \]

     \end{proof}

     \begin{lemm}[The norm of MD constraint vector]\label{lemma:md_norm}
         The following equality holds.
         \[
             \|\bm d\|^2=\frac{1}{|\mathcal I^+|}+\frac{1}{|\mathcal I^-|},
         \]
         where $\bm d$ is defined in \eqref{prob:single}
         %, i.e.
         %\[
    %         \bm d  = \frac{\bm s}{\bm1^\top \bm s}-\frac{(\bm 1 - \bm s) }{\bm 1^\top (\bm 1 - \bm s)}.
    %     \]

         \end{lemm}

     \begin{proof}\label{proof:md_norm}
         From the definition of $\bm s$, we can prove the lemma.
%         \begin{eqnarray}
%             \bm 1^\top \bm s &=& \|\bm s\|^2, \label{prop : s1} \\
%             \bm 1^\top (\bm 1 -\bm s) &=& \|\bm  1 - \bm s\|^2, \label{prop : s2}\\
%             \bm s^\top (\bm 1 -\bm s) &=& 0 \label{prop : s3}.
%         \end{eqnarray}
%         By using these equalities, we can prove the lemma.
         \begin{eqnarray*}
             \|\bm d\|^2 &=& (\frac{\bm s}{\bm1^\top \bm s}-\frac{(\bm 1 - \bm s) }{\bm 1^\top (\bm 1 - \bm s)})^\top
                 (\frac{\bm s}{\bm1^\top \bm s}-\frac{(\bm 1 - \bm s) }{\bm 1^\top (\bm 1 - \bm s)})\\
                 &=& \frac{\|\bm s\|^2}{(\bm1^\top \bm s)^2}+\frac{\|(\bm 1 - \bm s)\|^2}{(\bm1^\top (\bm 1 - \bm s))^2}(\because \bm s^\top (\bm 1 -\bm s) = 0)\\
                 &=& \frac{1}{\bm1^\top \bm s}+\frac{1}{\bm1^\top (\bm 1 - \bm s)}(\because  \bm 1^\top \bm s = \|\bm s\|^2, \bm 1^\top (\bm 1 -\bm s) = \|\bm  1 - \bm s\|^2)\\
                 &=& \frac{1}{|\mathcal I^+|}+\frac{1}{|\mathcal I^-|}
         \end{eqnarray*}
     \end{proof}

     Next, we prove Theorem \ref{theo:ipw_loss}.

     \begin{proof}\label{proof:ipw_loss}
         We can say
%         \begin{eqnarray*}
%             \|\bm h\|^2 &=& (\frac{\bm a}{\bm 1^\top \bm a}-\frac{\bm b}{\bm 1^\top \bm b})^{\top}(\frac{\bm a}{\bm 1^\top \bm a}-\frac{\bm b}{\bm 1^\top \bm b})\\
%                         &=& \frac{\|\bm a\|^2}{(\bm 1^\top \bm a)^2}+\frac{\|\bm b\|^2}{(\bm 1^\top \bm b)^2}\\
%                         &=& \frac{\sum_{i=1}^N (\frac{s_i}{z_i})^2}{(\sum_{i=1}^N \frac{s_i}{z_i})^2}+\frac{\sum_{i=1}^N (\frac{1-s_i}{1-z_i})^2}{(\sum_{i=1}^N \frac{1-s_i}{1-z_i})^2}\\
%                         &=& \frac{\sum_{i \in \mathcal{I}^+} (\frac{1}{z_i})^2}{(\sum_{i \in \mathcal{I}^+} \frac{1}{z_i})^2}+\frac{\sum_{i \in \mathcal{I}^-} (\frac{1}{1-z_i})^2}{(\sum_{i \in \mathcal{I}^-} \frac{1}{1-z_i})^2}\\
%                         &=& \frac{1}{\frac{(\sum_{i \in \mathcal{I}^+}\frac{1}{z_i})^2}{\sum_{i \in \mathcal{I}^+} (\frac{1}{z_i})^2}}
%                         +\frac{1}{\frac{(\sum_{i \in \mathcal{I}^-} \frac{1}{1-z_i})^2}{\sum_{i \in \mathcal{I}^-} (\frac{1}{1-z_i})^2}}
%         \end{eqnarray*}

         \begin{eqnarray*}
             \|\bm h\|^2 &=& \frac{1}{\frac{(\sum_{i \in \mathcal{I}^+}\frac{1}{z_i})^2}{\sum_{i \in \mathcal{I}^+} (\frac{1}{z_i})^2}}
                         +\frac{1}{\frac{(\sum_{i \in \mathcal{I}^-} \frac{1}{1-z_i})^2}{\sum_{i \in \mathcal{I}^-} (\frac{1}{1-z_i})^2}} \\
                         &\ge& \frac{1}{|\mathcal{I}^+|} + \frac{1}{|\mathcal{I}^-|} (\because \text{the Cauchy Schwartz inequality})\\
                         &=& \|\bm d\|^2 (\because \text{Lemma } \ref{lemma:md_norm})
         \end{eqnarray*}

%         We obtain the following inequality by using the Cauchy Schwartz inequality.
%
%         \[
%             \frac{(\sum_{i \in \mathcal I}  z_i)^2}{\sum_{i \in \mathcal I}  z_i^2} \le |\mathcal I|.
%         \]
%
%         By using this inequality, we can obtain
%         \begin{eqnarray*}
%             \|\bm h\|^2 &=& \frac{1}{\frac{(\sum_{i \in \mathcal{I}^+}\frac{1}{z_i})^2}{\sum_{i \in \mathcal{I}^+} (\frac{1}{z_i})^2}}
%             +\frac{1}{\frac{(\sum_{i \in \mathcal{I}^-} \frac{1}{1-z_i})^2}{\sum_{i \in \mathcal{I}^-} (\frac{1}{1-z_i})^2}}\\
%                         &\ge& \frac{1}{|\mathcal{I}^+|} + \frac{1}{|\mathcal{I}^-|}\\
%                         &=& \|\bm d\|^2 (\because \text{Lemma } \ref{lemma:md_norm})
%         \end{eqnarray*}
         Next, using Lemma \ref{lemma:optimal_loss}, $\|\bm h\|^2 \ge \|\bm d\|^2$, and the assumption $0 \le \bm h^\top \bm y \le \bm d^\top \bm y$, we can prove the Theorem \ref{theo:ipw_loss} as follows.
         \begin{eqnarray*}
             \|\bm X \bm w_{\mathrm{ipw}}^*-\bm y\|^2 &=& \frac{(\bm h^\top \bm y)^2}{\|\bm h\|^2}\\
                                         &\le& \frac{(\bm d^\top \bm y)^2}{\|\bm d\|^2}\\
                                         &=& \|\bm X \bm w_{\mathrm{single}}^*-\bm y\|^2
         \end{eqnarray*}
     \end{proof}

%    Theorem \ref{theo:ipw_loss} says that if the value estimated by the IPW estimator on a dataset (i.e. $\bm h^\top \bm y$)
%     is less than the value estimated by MD ($=\bm d^\top \bm y$),
%    the optimal value of the squared loss of FairCEE-IPW \eqref{prob:ipw} is less than that of Single MD \eqref{prob:single}.

     As with the previous theorem, we can show that Multi MD \eqref{prob:multi} also can achieve a smaller training error than  Single MD \eqref{prob:single}.
     \begin{theo}[Multi MD Loss]\label{theo:multi_loss}
         Let $\bm w_{\mathrm{single}}^*$, $\bm w_{\mathrm{multi}}^*$  be optimal solutions of Single MD \eqref{prob:single} and Multi MD \eqref{prob:multi}.
         Suppose that
         \[
            0 \le \frac{1}{K}\sum_{k=1}^K \bm d_{(k)}^\top \bm y_{(k)} \le \bm d^\top \bm y.
         \]
         Then we have
         \[
             \sum_{k=1}^K \|\bm X_{(k)} \bm w_{\mathrm{multi}}^* - \bm y_{(k)}\|^2 \le \|\bm X \bm w_{\mathrm{single}}^*-\bm y\|^2.
         \]
     \end{theo}

      %Here, we prove Theorem \ref{theo:multi_loss}.

      \begin{proof}\label{proof:multi_loss}
          Let $m_k$ be the MD of the $k$-th stratum (i.e. $m_k=\bm d_{(k)}^\top \bm y_{(k)}$).

          Let $\bm {\tilde d}$ be as follows.
          \[
                 \bm {\tilde d} = (\bm d_{(1)}^{\top},\bm d_{(2)}^{\top},...,\bm d_{(K)}^{\top})
          \]
          We define $\bm {\tilde X}$, $\bm {\tilde y}$ in the same way.
          By using this new notation and the original problem's constraint, the optimal solution of the original problem satisfies this following equality.

          \[
              \bm {\tilde d}^{\top}  \bm {\tilde X} \bm w_{\mathrm{multi}}^{*}= \sum_{k=1}^K \bm d_{(k)}^\top  \bm X_{(k)} \bm w_{\mathrm{multi}}^*=0
          \]
          The original problem can be rewritten as follows.
          \[
              \sum_{k=1}^K \|\bm X_{(k)} \bm w-\bm y_{(k)}\|^2 = \|\bm {\tilde X} \bm w - \bm {\tilde y} \|^2,\  \bm {\tilde d}^\top \bm {\tilde y}=\sum_{k=1}^K m_k
          \]

%          Then, from Lemma \ref{lemma:optimal_loss}, we can get
%
%          \[
%              \sum_{k=1}^K \|\bm X_{(k)} \bm w-\bm y_{(k)}\|^2 = \frac{(\sum_{k=1}^K m_k)^2}{\|\bm {\tilde d}\|^2}.
%          \]

          We get the following inequality:
          \begin{eqnarray*}
                                 \|\bm {\tilde d}\|^2  &=& \sum_{k=1}^K \frac{1}{|\mathcal I_k^+|} + \sum_{k=1}^K \frac{1}{|\mathcal I_k^-|}\\
                                   &\ge& \frac{K^2}{|\mathcal I^+|} + \frac{K^2}{|\mathcal I^-|} \\ %(\because \sum_{k=1}^K |\mathcal I_k^+|=|\mathcal I^+|, \sum_{k=1}^K |\mathcal I_k^-|=|\mathcal I^-|)\\
                                   &=& K^2\|\bm d\|^2(\because \text{Lemma }\ref{lemma:md_norm}).
          \end{eqnarray*}
          where $\mathcal I^+_k$ and $\mathcal I^-_k$ are the index subsets of $\mathcal I^{+}, \mathcal I^{-}$ for the $k$-th stratum.

          From this inequality and the assumption $0 \le \frac{1}{K}\sum_{k=1}^K m_k \le \bm d^\top \bm y$,
          we can prove Theorem \ref{theo:multi_loss} as follows.
          \begin{eqnarray*}
              \sum_{k=1}^K \|\bm X_{(k)} \bm w_{\mathrm{multi}}^*-\bm y_{(k)}\|^2 &=& \frac{(\sum_{k=1}^K m_k)^2}{\|\bm {\tilde d}\|^2} (\because \text{Lemma } \ref{lemma:optimal_loss})\\
                                                    &\le& \frac{(\frac{1}{K}\sum_{k=1}^K m_k)^2}{\|\bm d\|^2} \le \frac{(\bm d^\top \bm y)^2}{\|\bm d\|^2}\\
                                                    &=& \|\bm X \bm w_{\mathrm{single}}^*-\bm y\|^2.
          \end{eqnarray*}
      \end{proof}

     Theorem \ref{theo:multi_loss} says that if
     the mean of MD in each stratum corresponding to $\frac{1}{K}\sum_{k=1}^K \bm d_{(k)}^\top \bm y_{(k)}$ is smaller than MD on the entire dataset,
     the optimal value of the squared loss in Multi MD \eqref{prob:multi} is less than that of Single MD \eqref{prob:single}.

     When the dataset includes explanatory bias and the model of the propensity score is correct,
     by the definition of discrimination \eqref{def:def_disc},
     the causal effect estimated by the IPW estimator and stratification are less than MD.
     Therefore, these assumptions of Theorem \ref{theo:ipw_loss} and \ref{theo:multi_loss} are reasonable for our settings.
     From these two theorems, we can say both FairCEE-IPW and
     Multi MD give fair predictions and their losses are smaller than that of Single MD.
     We leave the analysis of FairCEE-DR to future work.

    \section{Numerical experiments}\label{chap:experiments}
     We obtained numerical results on synthetic data and real-world data in an attempt to answer the following research questions.
     \begin{itemize}
       \item (RQ1) Do the three CASEs mentioned in Section \ref{subsec : disadvantages} degrade the performance of Multi MD?
       \item (RQ2) Do FairCEEs work in the three CASEs?
       \item (RQ3) Do FairCEEs work on real-world data?
     \end{itemize}
     Since we know the true explanatory bias in the synthetic data (it is unknown in the real-world data), it is easy to understand the results.
     Hence, we first analyzed the behavior of FairCEEs and Multi MD in detail in an experiment using synthetic data.
     After that, we experimented using three real-world data in regression and binary classification tasks.

         \subsection{Dataset and Experimental Setup}\label{sec: dataset_setup}

     %                   The details of the synthetic and real-world data as well as the experimental setup are explained below.

             \subsubsection{Synthetic Data}
             Here we explain how to generate synthetic dataset for the regression task.
             We focus on only the synthetic dataset for the regression because it is not easy to generate a dataset for classification whose explanatory bias is known.
              We generated $100r\%$ of $N$ instances with $S=1$. The rest of the dataset had instances with $S=0$.
              If the instance had a positive sensitive feature (i.e. $S=1$), we drew the explanatory feature values $\bm X_e$ and non-explanatory feature values $\bm X_n$ from the following different multivariate normal distributions.
              \[
              \bm X_e \sim N(\bm \mu^+_{e},\bm I), ~~
              \bm X_n \sim N(\bm \mu^+_{n},\bm I)
              \]
              In the same way, if the instance had a negative sensitive feature (i.e. $S=0$),
              we drew samples from the following distributions.
              \[
              \bm X_e \sim N(\bm \mu^-_{e},\bm I), ~~ \bm  X_n \sim N(\bm \mu^-_{n},\bm I)
              \]
              where $\bm \mu^+_{e}, \bm \mu^-_{e} \in \mathbb R^{d_e}, \bm \mu^+_{n}, \bm \mu^-_{n} \in \mathbb R^{d-d_e}$ and $\bm I$ is an identity matrix.
              After sampling $\bm X_e$ and $\bm X_n$, we generate the  label $y_i$ of the $i$-th instance as follows.
              \[
              y_i = \bm w_e^\top \bm x_{i,e} + \bm w_n^\top \bm x_{i,n} + w_s s_i + \epsilon_i,
              \]
              where $\bm x_{i,e},\bm x_{i,n}, \epsilon_i$ are explanatory feature values,
              non-explanatory feature values and noise of the $i$-th instance respectively ($\epsilon_i \sim N(0,1)$).

              In this synthetic data, we can know the discrimination \eqref{def:causaleffect} and the true explanatory bias.
              \begin{eqnarray*}
%              \[
                \mbox{Discrim.} = &  \bm w_{n}^\top (\bm \mu^+_n-\bm \mu^-_n) + w_s \\
%              \]
%
 %             \[
                 \mbox{explanatory bias}=  &\bm w_{e}^\top (\bm \mu^+_e-\bm \mu^-_e)
%              \]
\end{eqnarray*}
              %$\bm w_e,\bm w_n, w_s,\bm \mu^+_n, \bm \mu^-_n, \bm \mu^+_e, \bm \mu^-_e$
              Since we know the values of these parameters used to generate the synthetic data,
              we can calculate Discrim. and the explanatory bias.

              In our experiments, we created the three CASEs of Section \ref{subsec : disadvantages} by modifying the parameters of the synthetic data.
              Unless otherwise specified, we generated the synthetic data with $N=2000$, $d=14$, $d_e=4$, and  $r=0.5$.
              We set $r=0.8$ when generating Imbalance CASE and $d=7$ and $d_e=2$ when generating Degenerate CASE.
              In Inferred CASE, we set $\bm \mu^+_{e}=1.5 $ and $\bm \mu^-_{e}=0$ (in other case, $\bm \mu^+_{e}=1.0$ by default).

             \subsubsection{Real-world Data}

% 大きいバージョン

%             \begin{table}
%                 \centering
%                 \begin{tabular}{|l|l|l|l|l|l|l|l|l|l|}\hline
%              datasets  &  $N$   & $d$  & $d_e$ & $S$ &  $(|\mathcal I^+|,|\mathcal I^-|)$ & $Y$  & MD & IPW & DR \\ \hline \hline
%              C\&C & 1994 & 99 & 4   & race  & (970, 1024)   &  Crime Rate  & 0.208  & 0.032 & 0.010  \\ \hline
%              LSAC & 20798 & 14 & 2   & race  & (19597,1201)   &  GPA  & 0.358 & 0.381 &  0.414 \\ \hline
%              COMPAS & 5855 &  14 & 2   & race  & (2914, 2941)   &  Recidivism  & 0.159 & 0.041 & 0.056 \\ \hline
%              ADULT & 30162 & 101  & 29   & gender  & (20380, 9782)   &  Income  & 0.225 & 0.205  & 0.063 \\ \hline
%              \end{tabular}
%              \caption{Statistics of the real-world datasets. $N$, $d$, and $d_e$  are  the number of instances, features and explanatory features,
%              respectively. MD, IPW, DR are estimated on each dataset.}
%              \label{tab:stats_real_world_data}
%             \end{table}

%% LSACいらないならこちらの方がいいかも

             \begin{table}
                 \centering
                % \small
                 \begin{tabular}{|l|l|l|l|l|l|}\hline
              datasets  &  $N$   & $d$   & $S$ &  $(|\mathcal I^+|,|\mathcal I^-|)$ & $Y$ \\ \hline \hline
              C\&C & 1994 & 99   & race  & (970, 1024)   &  Crime Rate    \\ \hline
%              LSAC & 20798 & 14    & race  & (19597,1201)   &  GPA  \\ \hline
              COMPAS & 5855 &  14    & race  & (2914, 2941)   &  Recidivism  \\ \hline
              ADULT & 30162 & 101    & gender  & (20380, 9782)   &  Income \\ \hline
              \end{tabular}
              \caption{Statistics of the real-world datasets.}
              \label{tab:stats_real_world_data}
             \end{table}

% 小さいバージョン
%        \begin{table}
%             \centering
%             \small %ココ
%                 \begin{tabular}{|l|l|l|l|l|l|}\hline
%                 datasets  &   $d$   & $(|\mathcal I^+|,|\mathcal I^-|)$  & MD & IPW & DR \\ \hline \hline
%                 C\&C & 99    & (970, 1024)    & 0.21  & 0.03 & 0.01  \\ \hline
% LSAC除外                 LSAC &  14    & (19597,1201)   & 0.36 & 0.38 &  0.41 \\ \hline
%                 COMPAS &  14    &  (2914, 2941)    & 0.16 & 0.04 & 0.06 \\ \hline
%                 ADULT & 101   & (20380, 9782)     & 0.23 & 0.21  & 0.06 \\ \hline
%                 \end{tabular}
%                 \caption{Statistics of the real-world datasets. $d$ is the number of features. MD, IPW, DR are estimated on each dataset.}
%                \label{tab:stats_real_world_data}
%         \end{table}

             We used three real-world datasets.
             We conducted experiments on two tasks, regression and binary classification.
             In the regression task, we used C\&C\footnote{https://archive.ics.uci.edu/ml/datasets/communities+and+crime}.
             %and LSAC \cite{lsac}.
             In the binary classification task, we used COMPAS \cite{compas} and ADULT\footnote{https://archive.ics.uci.edu/ml/datasets/adult}.
             Table \ref{tab:stats_real_world_data} shows the statistics of these real-world data after the preprocessing.

             The ADULT\footnote{The ADULT dataset is a weighted dataset. So this is not a good real-world dataset as typically used.
             However, we use it to verify the performance of FairCEEs since much research use this dataset \cite{Zafar2015FairnessCM,FairRegression,Calibration,Bengio2018,empirical_risk}.} and C\&C datasets are in the UCI repository \cite{UCI}.
             %, and the LSAC dataset is in \cite{komiyama}\footnote{Because the URL of the LSAC dataset has expired, we obtained it from the repository at https://github.com/jkomiyama/fairregresion of \cite{komiyama}}.
             We followed \cite{komiyama} when preprocessing the COMPAS datasets and followed \cite{Calders2013} when preprocessing the C\&C dataset. When we preprocessed the ADULT dataset, we removed missing data and binarized its categorical features and normalized its continuous features.
             We chose the explanatory features by following \cite{Calders2013} for C\&C and by following \cite{IPS_discovery} for ADULT. In the COMPAS dataset, we determined the explanatory features by conducting a dependency analysis following \cite{Calders2013}, whereby features that are highly correlated to both the target and the sensitive feature represent potential covariates.

             \subsubsection{Experimental Setup}
             We used FairCEE-IPW \eqref{prob:ipw} and FairCEE-DR \eqref{prob:dr} (respectively referred to as fcee\_ipw and fcee\_dr in the figures below).
             For comparison, we compared FairCEEs to methods using propensity score analysis\footnote{The methods using causal diagrams also deal with the explanatory bias.
             However, it is difficult to compare it fairly because the performance of methods using causal diagrams depends on the causal diagrams we assume.
             %             and this comparison is out of scope.
             }.
             We used the Multi MD \eqref{prob:multi} and Single MD \eqref{prob:single} models for comparison,
             The figures below refer to Multi MD with $K$ strata as multi\_$K$ and Single MD as single.
             The sensitive feature was not used in the prediction models. %We used logistic regression to estimate the propensity score.
             We conducted the experiments on the synthetic data 50 times and added error bars to the results.
             We ran our script once on each real-world dataset because solving classification problem by FariCEEs and Multi MD took more time than regression task.

         \subsection{Disadvantages of Multi MD (RQ1)}\label{sec: RQ1}

            \begin{figure}[t]
                 \centering
                 \includegraphics[width=0.8\linewidth,height=3.8cm]{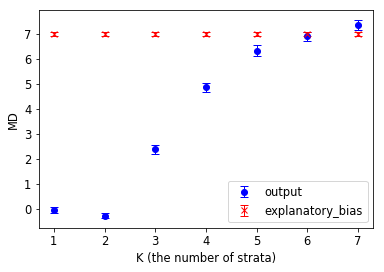}
                 \caption{Results of the MDs of Multi MD for different numbers of strata $K$.
                 By the definition of Discrim. \eqref{def:def_disc}, it is desirable that the MD of Multi MD is close to the explanatory bias (i.e. Discrim. is close to 0).
                 When $K$ is large enough ($K=5$,$6$ and $7$), Multi MD removes Discrim. \eqref{def:def_disc} correctly. However, too small a $K$ ends up removing explanatory bias and causes reverse discrimination.}
                 \label{fig:various_sn_md}
             \end{figure}

                 \begin{figure*}[t]
                   \centering
                     \begin{tabular}{c}

                       \begin{minipage}{0.33\hsize}
                           \centering
                           \includegraphics[width=\columnwidth,height=3.8cm]{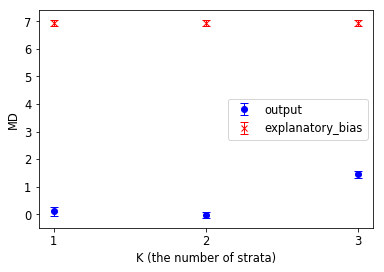}
                           \subcaption{Imbalance CASE}\label{fig:case1_RQ1}
                       \end{minipage}

                       \begin{minipage}{0.33\hsize}
                           \centering
                           \includegraphics[width=\columnwidth,height=3.8cm]{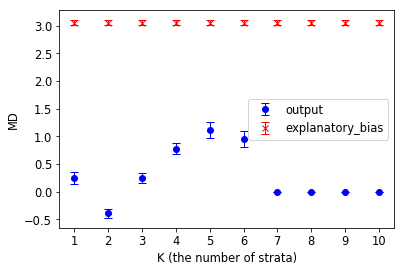}
                           \subcaption{Degenerate CASE}\label{fig:case2_RQ1}
                       \end{minipage}

                     \begin{minipage}{0.33\hsize}
                         \centering
                         \includegraphics[width=\columnwidth,height=3.8cm]{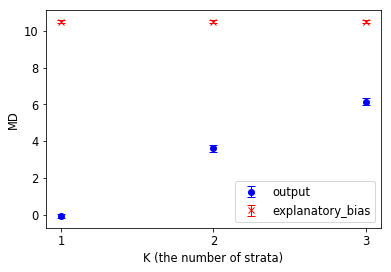}
                         \subcaption{Inferred CASE}\label{fig:case3_RQ1}
                     \end{minipage}
                     \end{tabular}
                     \caption{Results of experiments examining the disadvantages of Multi MD.
                      \eqref{fig:case1_RQ1} shows that imbalances with respect to the sensitive feature make it impossible to set $4 \le K$.
                      \eqref{fig:case2_RQ1} shows that when $K$ is larger than the number of features $d$ (=7), the solutions are zero vectors and the resultant MD is equal to 0.
                      \eqref{fig:case3_RQ1} shows that  it is impossible to set $4 \le K$ due to $z_i \approx s_i$.
                      In these cases, Multi MD incorrectly removes the explanatory bias.
                     }
                     \label{fig:RQ1}
                 \end{figure*}

         Now let us discuss the results of our experiment investigating the performance of Multi MD in the three CASEs described in Section \ref{subsec : disadvantages}.
         First, we show that Multi MD with a small $K$ incorrectly removes explanatory bias. Next, we show that it is difficult to set a large enough $K$ in the three CASEs.

         In Figures \ref{fig:various_sn_md} and \ref{fig:RQ1}, by the definition of Discrim. \eqref{def:def_disc}, it is desirable that the MD of Multi MD is close to the explanatory bias (i.e. Discrim. is close to 0).
         Figure \ref{fig:various_sn_md} shows that the MD of Multi MD with few strata is much less than the explanatory bias. This is because when $K$ is too small, each stratum has a range of propensity scores and some explanatory bias may remain in each stratum.

         Figure \ref{fig:case1_RQ1} indicated that it is impossible to split the dataset into more than three strata because of imbalances with respect to the sensitive feature.
         Figure \ref{fig:case2_RQ1} shows that MDs of Multi MD with $K \ge 7 = d$ are all zero. This is because degeneration causes the optimal solution of Multi MD to become the zero vector. When the number of strata $K$ is less than seven, the MD of the outputs is much less than the explanatory bias. This represents reverse discrimination because of few strata.
         As well as Figure \ref{fig:case1_RQ1}, Figure \ref{fig:case3_RQ1} shows that it is impossible to split the dataset into more than three strata because the estimated values of propensity scores are nearly equal to the sensitive features.
         Figures \ref{fig:case1_RQ1}, \ref{fig:case2_RQ1} and  \ref{fig:case3_RQ1} show that in the three CASEs, Multi MD can not increase $K$ appropriately and Multi MD mistakenly removes the explanatory bias.

          \subsection{Comparison between FairCEEs and Multi MD in the three CASEs (RQ2)}\label{sec: RQ2}

          \begin{figure*}[t]
            \centering
              \begin{tabular}{c}
                % 1
                \begin{minipage}{0.33\hsize}
                  \centering
                    \includegraphics[width=\columnwidth, height=3.8cm]{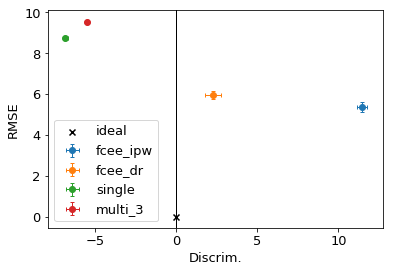}
                    \subcaption{Imbalance CASE}\label{fig:case1_RQ2}
                \end{minipage}

                % 2
                \begin{minipage}{0.33\hsize}
                  \centering
                    \includegraphics[width=\columnwidth, height=3.8cm]{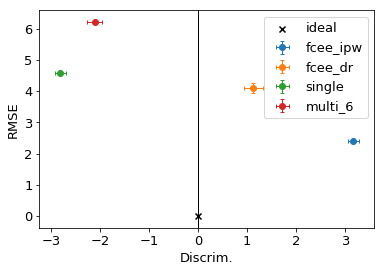}
                     \subcaption{Degenerate CASE}\label{fig:case2_RQ2}
                 \end{minipage}

                 % 2
                 \begin{minipage}{0.33\hsize}
                   \centering
                     \includegraphics[width=\columnwidth, height=3.8cm]{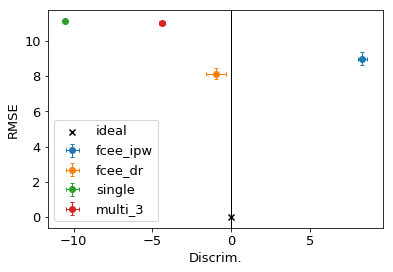}
                      \subcaption{Inferred CASE}\label{fig:case3_RQ2}
                \end{minipage}

              \end{tabular}
              \caption{Comparison of FairCEEs and Multi MD in the three CASEs. The horizontal axis represents Discrim. \eqref{def:def_disc}, and the vertical axis represents RMSE. The loss and Discrim. of FairCEE-DR are closer to 0 (the ideal point in this figure) than that of Multi MD. This means FairCEE-DR outperforms Multi MD in terms of loss and fairness in the three CASEs.}
              \label{fig:RQ2}

          \end{figure*}

              \begin{table}[t]
                \centering
                %\small %ココ
                    \begin{tabular}{|l|l|l|l|l|}
                    \hline
                          & Single MD             & Multi MD    & IPW             & DR             \\
                    \hline \hline
                    Imbalance CASE & 48.614 & 10.262  & 6.808 & {\bf 6.538} \\ \hline
                    Degenerate CASE & 9.418  & 0.793   & 0.221 & {\bf 0.217} \\ \hline
                    Inferred CASE & 110.459 & 14.870   & 21.361 & {\bf 1.603} \\
                    \hline
                    \end{tabular}
                    \caption{Squared error (SE) between each estimated value and Discrim. \eqref{def:def_disc} on the dataset in the three CASEs.
                    %(i.e. $\mbox{SE} = (*-\mbox{Discrim.})^2$, where $* \in \{\mbox{MD, IPW, DR}\}$.)
                    This table shows that the IPW and DR estimators estimate Discrim. more accurately than Single MD or Multi MD in Imbalance CASE and Degenerate CASE, and the DR estimator is the most accurate in all the CASEs including Inferred CASE.
                     %which are used in Single and Multi MD as the constraint.
                     }
                    \label{tab:causal_estimation_result}

              \end{table}

         Next we show that FairCEEs outperform Multi MD in terms of loss and fairness in the three CASEs.
         Figure \ref{fig:RQ2} is the loss-bias tradeoff graph in the three CASEs. The vertical axis represents RMSE, while the horizontal axis represents Discrim. \eqref{def:def_disc}. Note that, as mentioned in Section \ref{sec: dataset_setup}, we know the explanatory bias only in the experiment on synthetic data. Our goal here is to train a model that is fair at the expense of a small increase in loss. Accordingly, the loss and Discrim. should be close to 0 in Figure \ref{fig:RQ2} (this corresponds to the {\em ideal} point).
         Figure \ref{fig:RQ2} is the result of the comparison between FairCEEs and Multi MD in the three CASEs. Since it was found in Section \ref{sec: RQ1} that $K$ cannot be more than three in Imbalance CASE and Inferred CASE or more than seven in Degenerate CASE, we compare FairCEEs with Multi MD with $K=3$ in Imbalance CASE and Inferred CASE and $K=6$ in Degenerate CASE. Figure \ref{fig:RQ2} shows that the RMSE and Discrim. of FairCEE-DR are closer to 0 than those of Multi MD in the three CASEs.
         These results show that our FairCEE-DR outperforms Multi MD in terms of loss and fairness in the three CASEs.

         In addition, we calculated the squared error between each estimated value for Discrim. \eqref{def:def_disc} and the true value of Discrim. in the three CASEs
          to investigate how precisely the IPW and DR estimators estimate Discrim. \eqref{def:def_disc}.
         The result is in Table \ref{tab:causal_estimation_result}.
         This table indicates that the squared error between the IPW or DR estimator and Discrim.\footnote{Note that SE of Multi MD is defined as $\frac{1}{K}\sum_{k=1}^K(\mbox{MD}_k-\mbox{Discrim.})^2$, where $\mbox{MD}_k$ is the MD of the $k$-th stratum.} is much smaller than that of MD or Stratification used in Single or Multi MD in Imbalance CASE and Degenerate CASE.
         It also shows the DR estimator is the most accurate in all the CASEs including Inferred CASE.
         This is because the model for the propensity score is wrong and the IPW estimator can not accurately estimate Discrim., whereas the DR estimator works appropriately because of the robustness due to the estimators for $Y_1$ and $Y_0$.% described in Section \ref{sec: dr_model}.
         %In other words, the IPW and DR estimators can estimate Discrim. more accurately than other methods in the three CASEs.

                      \begin{figure*}[t]
                          \centering
                              \begin{tabular}{c}
                            % 1
                            \begin{minipage}{0.33\hsize}
                              \centering
                                \includegraphics[width=\columnwidth, height=3.8cm]{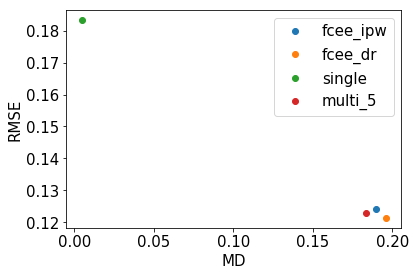}
                                \subcaption{C\&C}\label{fig:candc}

                            \end{minipage}

                            % 1
                            \begin{minipage}{0.33\hsize}
                              \centering
                                \includegraphics[width=\columnwidth, height=3.8cm]{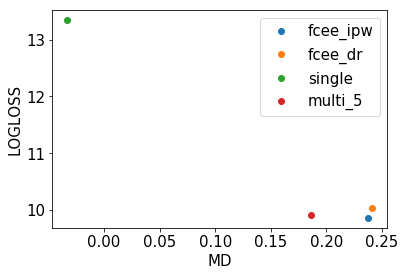}
                                \subcaption{COMPAS}\label{fig:compas}

                            \end{minipage}

                            \begin{minipage}{0.33\hsize}
                              \centering
                                \includegraphics[width=\columnwidth, height=3.8cm]{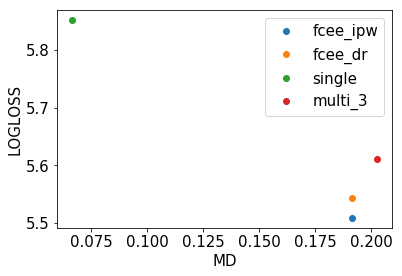}
                                \subcaption{ADULT}\label{fig:adult}

                            \end{minipage}

             %               \begin{minipage}{0.5\hsize}
             %                 \centering
             %                   \includegraphics[width=\columnwidth]{../images/lsac/loss_bias.png}
             %                   \subcaption{LSAC}\label{fig:lsac}

         %                   \end{minipage}

                          \end{tabular}
                          \caption{Results on real-world data.
                           \eqref{fig:candc} and \eqref{fig:compas} show that FairCEEs and Multi MD perform comparably.
                            It is difficult to determine which is superior because we do not know the true explanatory bias. \eqref{fig:adult} shows that the loss and MD of FairCEEs are lower than those of Multi MD.
                            This is because ADULT is a dataset with imbalanced sensitive features, as in Imbalance CASE in Section \ref{sec: RQ2}.
                            %(\ref{fig:lsac}) shows that the loss and MD of Single MD are very small. This is probably because the LSAC dataset includes very little explanatory bias and MD can quantify Discrim.\eqref{def:def_disc}.
                           }
                          \label{fig:real_data_result}

                      \end{figure*}

           \subsection{Evaluation in Real-world Data (RQ3)}\label{sec: RQ3}

         We visualized the results for the real-world data in the form of a loss-bias tradeoff graph as in Section \ref{sec: RQ2}. Note that we can not calculate Discrim. because we do not know the explanatory bias in the experiment on the real-world data. Thus, MD is the horizontal axis instead of Discrim. Here, because MD includes explanatory bias, it is not desirable that MD equal 0. Thus, we can not plot the {\em ideal} point, unlike in Figure \ref{fig:RQ2}. We used the logistic loss instead of RMSE in a binary classification task.

         %Note that the results for the LSAC dataset are in Appendix \ref{appendix: lsac_result}, because the LSAC dataset appears to include very little explanatory bias (see MD, IPW, DR in Table \ref{tab:stats_real_world_data}, where the estimated values of IPW and DR of the LSAC dataset are more than its MD and those of the other three datasets are much smaller than their MD).

         Figures \ref{fig:candc} and \ref{fig:compas} show that the MDs of the FairCEEs are larger than that of Multi MD, but the losses of FairCEE-DR and FairCEE-IPW are smaller than those of Multi MD in Figure \ref{fig:candc} and \ref{fig:compas}, respectively. It is difficult to say which is superior, as we do not know the true explanatory bias.
         Figure \ref{fig:candc} shows that the loss and MD of FairCEE-IPW are larger than those of Multi MD. This is probably because the model for the propensity score is not so good and the IPW estimator can not estimate the causal effect precisely. On the other hand, FairCEE-DR appropriately removes discrimination. This is because DR estimator can robustly estimate the causal effect by utilizing the estimators for $Y_1$ and $Y_0$.

         The ADULT dataset in Figure \ref{fig:adult} is imbalanced with respect to the sensitive feature (see Table \ref{tab:stats_real_world_data}). Therefore, for the same reason as in Imbalance CASE, we can not increase $K$ appropriately. Actually, Multi MD did not work when we set $K=4$. Therefore, we used Multi MD with three strata on the ADULT dataset.
         Figure \ref{fig:adult} shows that the loss and MD of FairCEEs are lower than those of Multi MD.
         This indicates that FairCEEs outperform Multi MD in terms of loss and fairness in the ADULT dataset, which belongs to Imbalance CASE.
         This result is consistent with the results of the experiment on synthetic data in Section \ref{sec: RQ2}.

     \section{Conclusion and Future Work}\label{chap:conclusion}
     We described FairCEEs based on IPW and DR estimators of causal effect that can remove discrimination while keeping explanatory bias.
     The proximal gradient method is applicable to FairCEEs with the logistic loss.
     We showed that FairCEE-IPW and Multi MD theoretically outperform Single MD in regression tasks.
     Our experiment on synthetic and real-world data indicated that FairCEE-DR outperforms Multi MD in cases where Multi MD does not work well due to too few strata.

%     \COMM{AT}{How about adding the following sentences?}
%     \COMM{HO}{I agree to your suggestion. Thank you.}
     By including a linear constraint of IPW=0 or DR=0, existing machine learning models can take into account reasonable fairness that eliminates explanatory bias. In this paper we only focused on regression and classification machine learning models, but this constraint can be used in a general framework. We would like to guarantee the validity of the general machine learning fairness model for future work.
      We have considered the setting with one binary sensitive feature in this paper. It may be possible to generalize most results of the paper to the settings of multiple sensitive features or a sensitive feature with multiple levels, i.e. $S \in \{ 0, 1, 2,...N_s\}$
      by considering causal effects in regard to each sensitive feature. We would like to investigate it further.

%     There are two other interesting directions of future study. One is the analysis of FairCEE-DR. The other is formulating a multi-class classification problem for FairCEEs.

%\COMM{AT}{Can we discuss the extension of our model for multiple sensitive features?}
%\COMM{HO}{notationのところにも書いたのですが、multiple sensitive featureにした場合はそれぞれのsensitive featureについての因果効果を制約に入れることで(複数制約になりますが)拡張可能ではないかと思います。}


%\begin{acks}
%To Robert, for the bagels and explaining CMYK and color spaces.
%\end{acks}

\bibliographystyle{ACM-Reference-Format}
\bibliography{main}

%\appendix

%\section{appendix}

\end{document}